%% file: main.tex
\documentclass[letterpaper, 10 pt, conference]{ieeeconf}  % Comment this line out if you need a4paper

\IEEEoverridecommandlockouts                              % This command is only needed if 
\usepackage{cite}
\usepackage{multicol}
\usepackage{algorithm2e}
\RestyleAlgo{ruled} 

\usepackage{amssymb}
\usepackage{amsmath}
 
\DeclareMathOperator*{\argmax}{argmax}

\usepackage{amsthm}
\newtheorem{definition}{Definition}
\newtheorem{theorem}{Theorem}

\newtheorem{lemma}{Lemma}
\newtheorem{remark}{Remark}

\usepackage{graphicx}
\usepackage{svg}
\usepackage{float}

\usepackage{multirow}

\makeatletter
\let\NAT@parse\undefined
\makeatother
\usepackage{url}
\usepackage[bookmarks=true]{hyperref}

\begin{document}

% paper title
\title{\LARGE \bf Temporally Robust Multi-Agent STL Motion Planning \\ in Continuous Time}
% or risk-aware

% You will get a Paper-ID when submitting a pdf file to the conference system
% \author{Author Names Omitted for Anonymous Review. Paper-ID [add your ID here]}

\author{Joris Verhagen$^{1}$, Lars Lindemann$^{2}$ and Jana Tumova$^{1}$
\thanks{This work was partially supported by the Wallenberg AI, Autonomous Systems and Software Program (WASP) funded by the Knut and Alice Wallenberg Foundation.}
\thanks{$^{1}$Joris Verhagen and Jana Tumova are with the Division of Robotics, Perception and Learning, School of Electrical Engineering and Computer Science, KTH Royal Institute of Technology, Stockholm, Sweden \{\tt\small jorisv, tumova\} @kth.se }
\thanks{$^{2}$Lars Lindemann is with the Thomas Lord Department of Computer Science, University of Southern California, Los Angeles CA, USA \tt\small llindema@usc.edu }
}

\maketitle
\thispagestyle{empty}
\pagestyle{empty}

\begin{abstract}
Signal Temporal Logic (STL) is a formal language over continuous-time signals (such as trajectories of a multi-agent system) that allows for the specification of complex spatial and temporal system requirements (such as staying sufficiently close to each other within certain time intervals). 
To promote robustness in multi-agent motion planning with such complex requirements, we consider motion planning with the goal of maximizing the temporal robustness of their joint STL specification, i.e. maximizing the permissible time shifts of each agent's trajectory while still satisfying the STL specification.
Previous methods presented temporally robust motion planning and control in a discrete-time Mixed Integer Linear Programming (MILP) optimization scheme. In contrast, we parameterize the trajectory by continuous B\'ezier curves, where the curvature and the time-traversal of the trajectory are parameterized individually. 
We show an algorithm generating continuous-time temporally robust trajectories and prove soundness of our approach.
Moreover, we empirically show that our parametrization realizes this with a considerable speed-up compared to state-of-the-art methods based on constant interval time discretization.% (up to TODO times).
\end{abstract}

% \IEEEpeerreviewmaketitle

\input{1_Introduction}
\input{2_Preliminaries}
\input{3_Problem_statement}
\input{4_Parametrization}
\input{5_Predicate_ATR}
\input{6_MILP_encoding}

\input{7_Results}
\input{8_Conclusions}

% \clearpage
\bibliographystyle{IEEEtran}
\bibliography{references}

\end{document}

%% file: 1_Introduction.tex
\section{Introduction}
\label{sec:introduction}
% GENERAL MOTIVATION
% Real-world robotic systems have to generate motion plans and controllers that are robust. 
Robustness of motion plans is often associated with the extent to which a signal can be spatially perturbed while still guaranteeing safety or desired performance. An example of this is specifying a robot to stay away from a wall where a door might suddenly open. Robustness is then increased by increasing the distance to the wall, promoting safety in such unforeseen events. 
We can however imagine scenarios where robustness is needed against temporal perturbations. This becomes especially apparent in multi-robot systems where robots need to collaborate or avoid each other, and where individual delays may result in task violation. 
Consider for example a robot-to-robot handover task where two robots should stay close enough for the handover to happen. 
During the execution of a pre-planned motion strategy, unmodeled dynamics, environmental factors, localization issues, and even onboard timing issues may cause one robot to be delayed, after which the other robot has already departed to fulfill the next part of its specification. An example of this is shown in Fig. \ref{fig:intro}.
% When unmodeled or unforeseen parts of the dynamics, external disturbances, or the environment are present during the execution of a pre-planned motion strategy, one of the robots might arrive later, where the previous robot has already left in order to satisfy the next part of its specification.
In these scenarios, it is clear that the temporal robustness of motion plans is of vital importance as it allows agents to arrive earlier or later than planned while still meeting a specification.
Specifically towards multi-robot systems, we are interested in the maximization of Asynchronous Temporal Robustness (ATR) \cite{lindemann2022temporal}, which considers the time shifts of individual agents within a predicate, i.e. agents being delayed or arriving early. We subsequently wish to generate feasible motion plans for multi-agent systems that maximize the ATR of a spatio-temporal specification. 

\begin{figure}[t!]
    \centering
    \includegraphics[width=0.35\textwidth]{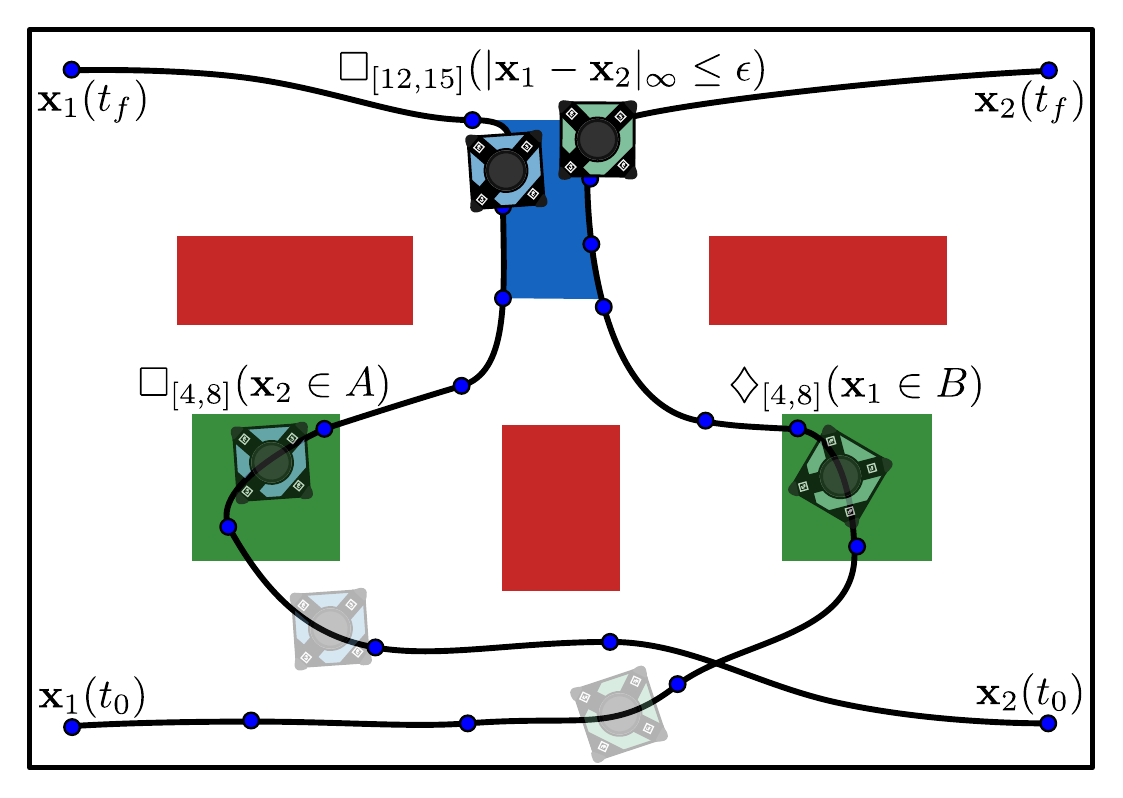}
    % \vspace{-0.5cm}
    \caption{An example of a multi-agent STL specification. Two robots are required to be close enough to each other for a predetermined amount of time (blue) in addition to an individual spatio-temporal requirement (green). To ensure the time robustness of the collaborative task, we need to consider time shifts on both trajectories.}
    \label{fig:intro}
    \vspace{-0.5cm}
\end{figure}

% PROBLEM I'M SOLVING
% Conversely, we utilize an efficient B\'ezier curve representation that allows continuous-time temporal robustness with lower computational complexity.
%In this work, we present the first implementation of continuous-time temporal robustness in motion planning by utilizing an efficient B\'ezier parametrization. We additionally realize this with a significant speed-up compared to methods that use a constant interval time discretization.

\subsection{Contributions}
\label{ssec:contributions}
% THINGS I HAVE ACHIEVED
In this work, we consider double-integrator, continuous-time multi-agent systems subjected to a fragment of bounded time STL with linear predicates. We develop upon existing works that consider discrete-time temporal robustness via the constant-interval discretization of a system's trajectory. 
% We then, as opposed to existing work, optimize the continuous-time temporal robustness. 
However, in contrast to state of the art, we optimize the continuous-time temporal robustness. 
The contributions of our work are summarized as follows:
\begin{itemize}
    \item We present an efficient multi-agent B\'ezier curve parametrization of the system's trajectory that allows for reasoning over continuous space and time.
    We empirically demonstrate that this continuous time may lead to the theoretically optimal temporal robustness of a motion plan.
    % We empirically demonstrate that continuous-time temporal robustness may exceed that of discretization approaches.
    \item We propose a Mixed-Integer Linear Program (MILP) that finds a continuous multi-agent motion plan that is robust to asynchronous time shifts, and we show its soundness.
    \item We provide a theoretical analysis of computational complexity and empirical evidence of speed-up compared to methods based on time discretization.
    % Our method is compared to existing ones in single- and multi-agent scenarios.
\end{itemize}

\subsection{Related Work}
\label{ssec:related_work}
% RELATED WORK AROUND lARS' PAPERS
% RELATED WORK AROUND BEZIER FORMULATION
To formulate tasks in both spatial and temporal domains, temporal logics such as Linear Temporal Logic (LTL) have been employed. 
Although LTL allows temporal ordering, it does not allow the user to reason over time quantitatively.
Addressing these limitations are methods such as Metric Temporal Logic and Signal Temporal Logic (STL) \cite{maler2004monitoring}. 
As STL allows reasoning over continuous-time signals, we are interested in maximizing the temporal robustness of specifications that consider a fragment of this temporal logic. 

As STL can account for continuous signals, it allows for direct reasoning over the robustness of spatio-temporal missions. 
Many existing works consider spatial robustness of STL specifications in motion planning \cite{raman2014model,mehdipour2019average,vahs2023risk} based on the definition of robust semantics in \cite{fainekos2009robustness} and \cite{donze2010robust}. 
In time-critical systems, however, it is not enough to consider spatial robustness. 
In collaborative multi-robot systems, for example, we wish to consider maximizing temporal robustness, which has been given less attention. 
Work in \cite{lin2020optimization} defines temporal robustness on predicate shifts on a small fragment of STL. 
Full (time-bounded) STL with predicate shifts is considered in \cite{rodionova2021time,rodionova2022temporal}. 
Although it does not consider individual shifts of signals within a predicate, the work in \cite{rodionova2022combined} combines left- and right-temporal robustness (accommodating advancements and delays respectively) to generate motion plans that are robust to time shifts in both directions. 
Asynchronous Temporal Robustness (ATR), formally defined in \cite{lindemann2022temporal}, is used as a constraint in \cite{yu2023efficient}. 
In a similar vein, the authors in \cite{sahin2017synchronous,sahin2019multirobot} present counting LTL (cLTL) which is able to bound asynchronous shifts in discrete space and time. Related is also the work in \cite{vasile2017time} which concerns temporal relaxations.
In contrast, we aim to not just constrain the ATR, but to directly maximize it.
Further, all these works rely on a constant-time discretization of the trajectory, fundamentally limiting the resolution of both the trajectory and its temporal robustness while also requiring a high number of variables for long-horizon missions. To this end, we represent trajectories with independent space- and time parametrization.

Multi-agent STL motion planning with independent space and time parametrization was introduced in \cite{sun2022multi} where piecewise linear segments are concatenated to satisfy qualitative semantics of STL specifications. 
The trajectory formulation in this work generalizes this by using a B\'ezier parametrization from \cite{marcucci2022motion} in a multi-agent context that allows for non-constant velocity segments, continuous differentiability, and, as we will show, quantitative temporal semantics of STL. 
% The B\'ezier parametrization relies on redefining the coupling between the trajectory and time via the introduction of a phasing variable \cite{verscheure2009time}.
% \\\\
% The rest of this work is organized as follows: Sec. \ref{sec:preliminaries} states preliminary information on B\'ezier curves and STL, in Sec. \ref{sec:problem_statement} we state the problem. 
% In Sec. \ref{sec:bezier_parametrization}, we define the B\'ezier parametrization, followed by the a new definition of ATR in Sec. \ref{sec:rATR}. Sec. \ref{sec:MILP_encoding} addresses how this definition is implemented in a MILP. 
% We then present simulation results in Sec. \ref{sec:results} and end our work in Sec. \ref{sec:conclusions} with conclusions and future research directions.

%% file: 2_Preliminaries.tex
\section{Preliminaries}
\label{sec:preliminaries}
Let $\mathbb{R}$ and $\mathbb{N}$ be the set of real and natural numbers including zero, respectively, and let $\mathbb{R}_{\ge 0}$ denote the set of real, non-negative numbers.
Let $\mathbb{B} = \{\top,\bot\}$, where $\top$ and $\bot$ indicate the Boolean true and false values.
Additionally, let $\mathbf{x}(t) \in \mathbb{R}^{\text{dim}\cdot n}$ be the state of a multi-agent system at time $t \in \mathbb{R}_{\ge 0}$ %of length $p \cdot q$, 
where $\mathbf{x}_k(t) \in \mathbb{R}^{\text{dim}}, k \in \{1,\ldots,n\}$ is regarded as the state of agent $k$, and $n$ is the number of agents. 
The mapping $\mathbf{x}: \mathbb{R}_{\ge 0} \rightarrow \mathbb{R}^{\text{dim}\cdot n}$ is a \emph{signal}. %and 
Lastly, let $\mathbb{R}^{p\times q}$ be a $p$ by $q$ matrix of real numbers.

\begin{figure}[t!]
    \centering
    \includegraphics[width=0.5\textwidth]{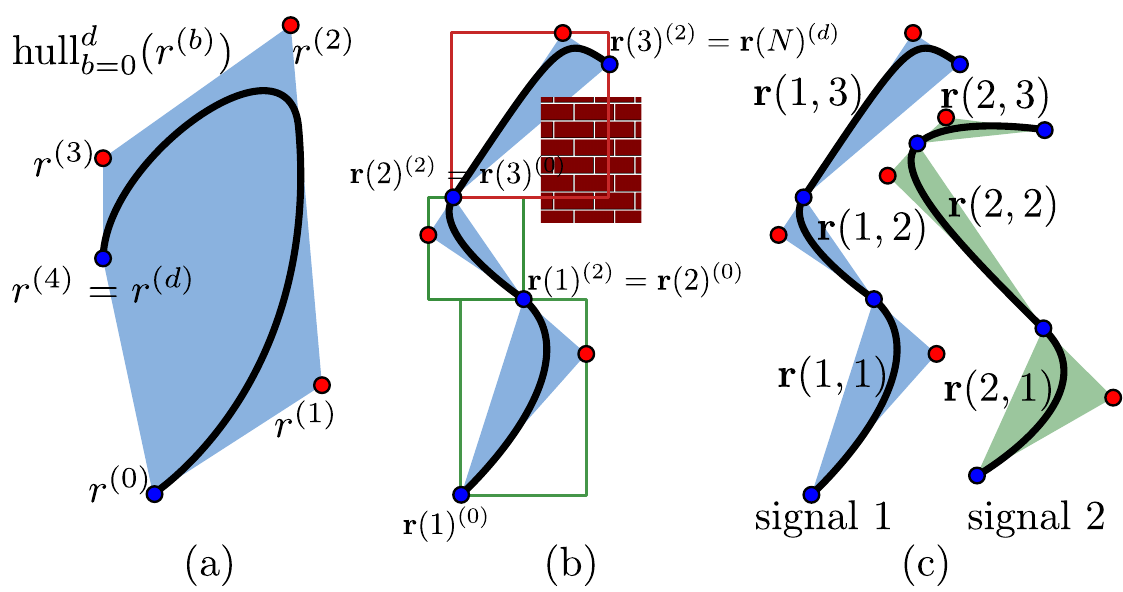}
    \caption{(a) A B\'ezier curve with its control points and its convex hull (in light blue). $r^{(b)}$ indicates the $b$'th control point of $r$. (b) A concatenation of 3 B\'ezier curves with their convex hulls. Consideration of the bounding box for obstacle avoidance. $\mathbf{r}_{(2)}$ is not collision-free according to the conservativeness of Eq. \eqref{eq:constraint:w_free}. (c) Multi-agent B\'ezier parametrization indicating segments of $\mathbf{r}_k$ for $k=1$ and $k=2$.}
    \label{fig:beziers}
    \vspace{-0.5cm}
\end{figure}

\subsection{B\'ezier curves}
\label{ssec:bezier_curves}
Our approach relies on B\'ezier curves %, a signal parametrization 
that have found many use cases in motion planning of autonomous systems \cite{qian2016motion,marcucci2022motion}. 
A B\'ezier curve is represented by a polynomial equation and parameterized by a finite number of control points, its decision variables. %, its control points, and is represented by a polynomial equation. 
Namely, a B\'ezier curve $r$ of degree $d$ is constructed and evaluated using the summation of its $d$ Bernstein polynomials multiplied with their respective $d$ control points according to
\begin{equation}
    r(s) := \sum_{b=0}^{d}\begin{pmatrix}d \\ b\end{pmatrix}(1-s)^{d-b}s^b \cdot r^{(b)}, \quad s\in[0,1],
\end{equation}
where $s$ is the phasing parameter, $\begin{pmatrix}d \\ b\end{pmatrix}(1-s)^{d-b}s^b$ is the $b$'th Bernstein polynomial, %$d$ is the degree of the B\'ezier curve, , 
and $r^{(b)}$ is the $b$'th control point. 
An example of a B\'ezier curve with its control points is shown in Fig. \ref{fig:beziers}a.
Although B\'ezier curves are nonlinear in nature, the following properties allow us to %some advantageous properties of these curves allow us to 
reason over their convex over-approximation: %We rely on :
\begin{enumerate}
    \item \textit{Convex hull:} the curve $r(s)$ is entirely contained within the convex hull generated by its control points $r^{(b)}$, $r(s) \in \mathrm{hull}_{b=0}^{d}(r^{(b)}) \quad \forall s \in [0,1]$, shown in Fig. \ref{fig:beziers}a.
    \item \textit{Endpoint values:} the curve $r(s)$ starts at the first control point $r^0$ as $s=0$ and ends at the last control point $r^{(d)}$ as $s=1$. % can be linearly constrained as they equal the control points.
    \item \textit{Derivatives:} the derivative $\dot{r}(s) = \frac{dr(s)}{ds}$ is a linear combination of the control points of $r(s)$ and is again a B\'ezier curve of degree $d-1$ with control points $\dot{r}^{(b)} = d\cdot(r^{(b+1)} - r^{(b)})$ for $b=\{0,...,d-1\}$.
\end{enumerate}
Subsequently, the derivatives of the start- and end-point, $\dot{r}^{(0)}$ and $\dot{r}^{(d-1)}$, are linear combinations of the control points of $r(s)$ and the convex hull of $\dot{r}$ is formed as a linear combination of its control points.
% the endpoints of the derivative $\dot{r}(s)$ can be explicitly constrained using linear constraints, and the velocity curve $\dot{r}(s)$ is contained within the convex hull generated by a linear combination of its control points $r^k$.
This means that smooth, kinematically feasible, and collision-free trajectories can be constructed through Linear Programming by returning a finite set of control points of a B\'ezier curve.
% Despite being parameterized only by a finite number of control points,  B\'ezier curves allow for the construction of a continuously differentiable, kinematically feasible, and collision-free trajectories using Linear Programming.
%Using a B\'ezier parametrization based on \cite{marcucci2022motion}, introduced in Sec. \ref{sec:bezier_parametrization}, we will show how these properties enable efficient multi-agent motion planning under spatio-temporal missions.

\subsection{Signal Temporal Logic}
\label{ssec:signal_temporal_logic}
Signal Temporal Logic (STL) \cite{maler2004monitoring} %extends the classical propositional logic by 
considers real-valued signals making it a powerful tool for specifying and verifying properties of dynamical systems.
Let us first define a fragment of STL that we use in this work that specifies desired properties of $n$-dimensional, finite, continuous-time signals $\mathbf{x}:\mathbb{R}_{\ge 0} \rightarrow X \subseteq \mathbb{R}^n$.
%, continuous-time 

\begin{definition}[Fragment of Signal Temporal Logic]
\label{def:stl}
    Let bounded time intervals $I$ be in the form $[t_1,t_2]$, where, for all, $I \subset \mathbb{R}_{\ge 0}$, $t_1,t_2 \in \mathbb{R}_{\ge 0}, t_1 \leq t_2$. 
    % $[\tau+t_1, \tau+t_2]$ is denoted by $\tau+I$, $\tau \in \mathbb{R}$. 
 %   We consider an %Let the signal state at time $t$ be $\mathbf{x}(t) \in X$, l
    Let $\mu:X \rightarrow \mathbb{R}$ be a linear real-valued function, and let $p:X \rightarrow \mathbb{B}$ be a linear predicate defined according to the relation $p(\mathbf{x}) := \mu(\mathbf{x}) \ge 0$. The set of predicates is denoted $\textit{AP}$. %The predicate $p$ therefore defines the set $\{\mathbf{x}\in X \mid \mu(\mathbf{x}) \ge 0\}$. 
    We consider a fragment of STL, recursively defined as
    % The syntax of a limited fragment of STL, only considering conjunctions and disjunctions of {\em Always} and {\em Eventually} operators, is then recursively defined as    
    % \begin{equation}
    % \label{eq:stl_fragment}
    %     \phi ::= \top \mid p \mid \phi_1 \wedge \phi_2 \mid \phi_1 \lor \phi_2 \mid \diamondsuit_{I} \phi \mid \Box_{I} \phi,
    % \end{equation}
    \begin{equation}
    \label{eq:stl_fragment}
        \begin{aligned}
            % \pi &::= \top \mid p \\
            \psi &::= p \mid \diamondsuit_{I}p \mid \Box_{I} p \\
            \phi &::= \psi \mid \psi_1 \land \psi_2 \mid \psi_1 \lor \psi_2 
        \end{aligned}
    \end{equation}
    % where $\top$ is the Boolean \textit{True} constant, 
    where $p \in \textit{AP}$. 
    The symbols $\wedge$ and $\lor$ denote the Boolean operators for conjunction and disjunction, respectively; and $\diamondsuit_{I}$ and $\Box_{I}$ denote the temporal operator {\em Eventually} and {\em Always}.% over bounded interval $I$ respectively.
\end{definition}
% Other Boolean operations are defined using the conjunction and negation operators to enable the full expression of propositional logic.
% The additional boolean operator {\em Or} and temporal operators {\em Eventually} and {\em Always} are defined as $\phi_1 \lor \phi_2 \equiv \neg (\neg \phi_1 \land \neg \phi_2)$, $\diamondsuit_{I} \phi \equiv \top \mathcal{U}_{I} \phi$, and $\Box_{I} \phi \equiv \neg \diamondsuit_{I} \neg \phi$,  respectively. 
%The set of all STL formulae over signals in $\Sigma$ is denoted by $\Phi^{\Sigma}$. 
% \\\\
In the STL fragment, we consider the bounded {\em Always} and {\em Eventually} operators, (requiring to have a predicate hold for all time $t \in I$ and for any time $t \in I$, respectively) and allow conjunctions and disjunctions of any combinations of these. 
% We therefore define the boolean satisfaction and robustness of STL with respect to $t=0$.
%We will later see that to enable an efficient encoding, we use B\'ezier curves in space and time. This motivates the limited fragment as quantitative time semantics over segments that span a non-constant time lead to minimax problems.
% We will later motivate the consideration of this limited fragment.
% This will be further elaborated upon in Sec. \ref{sec:conclusions}.

\subsection{Time-Robust Semantics of STL}
The qualitative semantics of STL indicate whether a signal $\mathbf{x}$ satisfies a specification $\phi$ or not. %, and is recursively computed via evaluation on its predicates. 
However, we wish to additionally specify how robustly a signal satisfies the specification. To that end, we use a version of  Asynchronous Temporal Robustness (ATR) semantics inspired by \cite{lindemann2022temporal}. 
We call the variant \emph{recursive ATR}; it defines the allowed  time shifts on the states of the agents in a multi-agent system for a predicate $p$, and further recursively for the fragment of STL introduced in Def. \ref{def:stl}. Intuitively, the ATR indicates to what extent the state $\mathbf{x}_k$ of each agent $k$ can be shifted in time asynchronously with respect to the states of other agents, while still having the entire system's state $\mathbf{x}$ satisfying the specification $\phi$. This property becomes especially interesting in multi-agent systems, where agents are subject to individual time-shifts, but a joint specification. For example, for the system in Fig. \ref{fig:intro}, such a joint specification is "{\em Always, between 12 and 15 seconds, robot 1 and robot 2 should be $\epsilon$-close}". % occurs, this becomes critical as agents are subjected to individual time-shifts.
%As such, we state a new definition of ATR, predicate ATR, which is inspired by the ATR in \cite{lindemann2022temporal}. 
%We call the variant \emph{predicate ATR}; defining the allowed asynchronous time shifts on the states of the agents for a predicate $p$. 

% and can be recursively computed for $\phi$.
First, let $\mathcal{S}_p = \{k_1,\ldots,k_m\}$ be the set of indices of all agents that influence the truth value of predicate $p$. %for which the state $\mathbf{x}_k(t), k \in \mathcal{S}_p$ appears in predicate $p$. 
Let the relevant state vector for predicate $p$ be defined as $\mathbf{x}_p(t) = [\mathbf{x}_{k_1}(t),\ldots,\mathbf{x}_{k_m}(t)]^T, \forall k_i \in \mathcal{S}_p$. %, containing the states of agents that appear in predicate $p$.
% Predicate ATR considers time shifts on the state of individual agents in a predicate $p$. 
With respect to the example in Fig. \ref{fig:intro}, $\mathcal{S}_p = \{1,2\}$ and $\mathbf{x}_p(t) = [\mathbf{x}_1(t),\mathbf{x}_2(t)]^T$.

\begin{definition}[Recursive ATR]
\label{def:predATR}
    Let us define the time-shifted state that is relevant to predicate $p$ as $\mathbf{x}_{p,\bar{\kappa}}(t)$, with time shift $\bar{\kappa} := \big\{\kappa_k \big\rvert k \in \mathcal{S}_p\big\} \in \mathbb{R}^m, m = |\mathcal{S}_p|$ as 
    \begin{equation}
    \label{eq:kappa_bar}
        \mathbf{x}_{p,\bar{\kappa}}(t) := [\mathbf{x}_{k_1}(t+\kappa_{k_1}),\ldots,\mathbf{x}_{k_m}(t+\kappa_{k_m})]^T \quad \forall k \in \mathcal{S}_p.
    \end{equation}
    where $\mathbf{x}_{p,\bar{\kappa}}(t) \in \mathbb{R}^{\text{dim}\cdot m}$. Regardless of the dimensionality of the workspace, an agent is subjected to a single, unique, time shift; $\bar{\kappa} \in \mathbb{R}^m$ while $\mathbf{x}_{p,\bar{\kappa}}(t) \in \mathbb{R}^{\text{dim}\cdot m}$.

 %We consider each agent in predicate $p$ and let 
    The recursive ATR first defines the maximum extent the state of each agent $l \in \mathcal{S}_p$ %including $k$, 
    may be shifted in time asynchronously such that the predicate still holds.
    \begin{multline}
    \label{eq:ATR^p}
        \bar\theta_{p}(\mathbf{x},t) :=  \bar\chi_{p}(\mathbf{x}_p,t) \cdot \\
        \max\left\{\tau \ge 0 : \begin{array}{c} \kappa_l \in [-\tau,\tau], \forall l \in \mathcal{S}_p, \\ \bar\chi_{p}(\mathbf{x}_{p,\bar{\kappa}},t) = \bar\chi_{p}(\mathbf{x}_p,t)\end{array} \right\},
    \end{multline}
    where $\bar\chi^p$ is the characteristic function, defined as
    \begin{equation}
        \bar\chi_p(\mathbf{x},t) := \begin{cases}
            1, &\text{if} \enspace \mu(\mathbf{x},t) \geq 0 \\%, \forall i : \forall l \in \mathcal{S}_p, l \neq k \\
            -1, &\text{else}
        \end{cases}
    \end{equation}

%\begin{remark}
    % Although the state of an agent, $\mathbf{x}_k(t) \in \mathbb{R}^{\text{dim}}$ might consist of multiple sub-signals, e.g. two sub-signals for a two-dimensional environment, an agent is still subjected to a single delay or advancement; $\bar{\kappa} \in \mathbb{R}^m$ while $\mathbf{x}_{p,\bar{\kappa}}(t) \in \mathbb{R}^{\text{dim}\cdot m}$
%\end{remark}

We additionally define ATR recursively for the temporal and Boolean operators of the STL fragment in Eq. \eqref{eq:stl_fragment}.
\begin{equation}
\label{eq:theta_ATR_temp_bin}
\begin{aligned}
    \bar\theta_{\Box_I p}(\mathbf{x}) &:= \min_{t' \in I}(\bar\theta_{p}(\mathbf{x},t')), \\
    \bar\theta_{\diamondsuit_I p}(\mathbf{x}) &:= \max_{t' \in I}(\bar\theta_{p}(\mathbf{x},t')), \\
    \bar\theta_{\phi_1 \land \phi_2}(\mathbf{x}) &:= \min(\bar\theta_{\phi_1}(\mathbf{x}),\bar\theta_{\phi_2}(\mathbf{x})), \\
    \bar\theta_{\phi_1 \lor \phi_2}(\mathbf{x}) &:= \max(\bar\theta_{\phi_1}(\mathbf{x}),\bar\theta_{\phi_2}(\mathbf{x})).
\end{aligned}
\end{equation}
obtaining the ATR of specification $\phi$ as $\bar\theta_{\phi}$.
\end{definition}

A non-negative ATR indicates the extent to which the state of all agents may be time-shifted asynchronously while still satisfying the predicate $p$. In contrast, negative ATR indicates the minimum extent to which the state of each agent needs to be shifted in time to change from violation of $p$ to its satisfaction.

\begin{remark}
\label{remark:remark_th}
    When a predicate is dependent on a single agent, e.g. $\Box_{[4,8]}(\mathbf{x}_2 \in A)$ as per Fig. \ref{fig:intro}, the Recursive ATR collapses to the combined left- and right temporal robustness of \cite{rodionova2022combined}, defining temporal robustness via the maximum permissible shift of the temporal interval $I$. 
\end{remark}

\begin{lemma}
    For the ATR $\bar\theta_{\phi}$, it follows that $\max(|\kappa_{k_1}|,\ldots,|\kappa_{k_m}|) \leq |\bar\theta_{\phi}| \implies \bar\chi_{\phi}(\mathbf{x}_{\bar\kappa}) = \bar\chi_{\phi}(\mathbf{x})$.
\end{lemma}
\begin{proof}[Proof Sketch]
    Notice that Def. \ref{def:predATR} states that $\bar\theta_p(\mathbf{x},t) = c$ indicates that $\forall \kappa_{k_1},\hdots,\kappa_{k_m} \in [-c,c]$ we have that the shifted signal $\mathbf{x}_{p,\bar\kappa}(t)$ also satisfying predicate $p$. 
    $\bar\theta_{\Box_I p}(\mathbf{x}) = c$ subsequently states that $\forall t \in I, \bar\theta_p(\mathbf{x},t) \geq \bar\theta_{\Box_I p}(\mathbf{x}) = c$ ensuring ATR for all times $t \in I$ of at least $c$. Conversely, $\bar\theta_{\diamondsuit_I p}(\mathbf{x}) = c$ states that $\exists t\in I, \bar\theta_p(\mathbf{x},t) = \bar\theta_{\Box_I p}(\mathbf{x}) = c$ ensuring ATR for a time $t \in I$ of $c$.
    The Boolean operators are trivial. Given the fragment in Eq. \eqref{eq:stl_fragment} this concludes the proof sketch.
\end{proof}

%% file: 3_Problem_statement.tex
\section{Problem Statement}
\label{sec:problem_statement}
Consider a double-integrator, continuous-time multi-agent dynamical system of the form 
\begin{equation}
\label{eq:double_integrator}
    % \dot{x}(t) = Ax(t) + Bu(t),
    \ddot{\mathbf{x}}(t) = \mathbf{u}(t), \quad \mathbf{x}(t_0) = \mathbf{x}_{t_0}, \dot{\mathbf{x}}(t_0) = \dot{\mathbf{x}}_{t_0},
\end{equation}
where $\mathbf{x} \in \mathbb{R}^{{\text{dim}}\cdot n}$ is a stacked vector of the states $\mathbf{x}_k \in \mathbb{R}^{\text{dim}}$ of the agents, with $\text{dim}$ being the dimension of the workspace and $n$ the number of agents. The system is subject to an initial- and final state $\mathbf{x}_{t_0}$ and $\mathbf{x}_{t_f}$ at initial- and final time $t_0$ and $t_f$, velocity constraints $\dot{\mathbf{x}} \in \mathcal{V} = [\bar{\dot{\mathbf{x}}},\underline{\dot{\mathbf{x}}}]$ for all agents, and a spatial-temporal mission $\phi$ expressed from the fragment of STL in Eq. \eqref{eq:stl_fragment}. 
Additionally, consider the workspace $\mathcal{W}$ as a convex polygon with convex polygon obstacles $\mathcal{W}_{obs}$, defining the free workspace as $\mathcal{W}_{free} = \mathcal{W} \setminus \mathcal{W}_{obs}$. 
% We aim to find a collision-free motion plan, that maximizes the left, right, or asynchronous temporal robustness $\theta^{(\pm)}_{\phi}$ of the specification.
We aim to find a collision-free motion plan that maximizes the ATR $\bar{\theta}^{*}_{\phi}$ from Def. \ref{def:predATR}.
As such, we wish to obtain a continuous-time control input such that the corresponding trajectory $\mathbf{x}^*(t)$ satisfies the specification $\phi$ with the ATR$\bar{\theta}^{*}_{\phi}$. 
Subsequently, we define the optimization problem
% \begin{equation}
% \label{eq:MILP_problem}
% \begin{aligned}
% \max_{\mathbf{r},\mathbf{h}} \quad & \bar{\theta}_{\phi}(\mathbf{x})\\
% \textrm{s.t.} \quad & \ddot{\mathbf{x}}(t) = \mathbf{u}(t), \forall t \in [t_0,t_f],\\
%                     & \bar{\theta}_{\phi}(\mathbf{x}) > 0,\\
%                     & \mathbf{x}(t_0) = \mathbf{x}_{t_0}, \mathbf{x}(t_f) = \mathbf{x}_{t_f},\\
%                     & \mathbf{x}_k(t) \in \mathcal{W}_{free}, \forall t \in [t_0,t_f], \forall k \in \{0,\ldots,n\},\\
%                     & \dot{\mathbf{x}}_k(t) \in \mathcal{V}, \forall t \in [t_0,t_f], \forall k \in \{0,\ldots,n\}.
% \end{aligned}
% \end{equation}
% \begin{equation}
% \label{eq:MILP_problem}
\begin{align}
\argmax_{\mathbf{u}} \quad & \bar{\theta}_{\phi}(\mathbf{x}) \label{eq:MILP_problem_cost}\\
\textrm{s.t.} \quad & \ddot{\mathbf{x}}(t) = \mathbf{u}(t), \forall t \in [t_0,t_f], \tag{8a}\label{eq:MILP_problem_c1}\\
                    & \bar{\theta}_{\phi}(\mathbf{x}) > 0, \tag{8b}\label{eq:MILP_problem_c2}\\
                    & \mathbf{x}(t_0) = \mathbf{x}_{t_0}, \mathbf{x}(t_f) = \mathbf{x}_{t_f}, \tag{8c}\label{eq:MILP_problem_c3}\\
                    & \mathbf{x}_k(t) \in \mathcal{W}_{free}, \forall t \in [t_0,t_f], \forall k \in \{1,\ldots,n\}, \tag{8d}\label{eq:MILP_problem_c4}\\
                    & \dot{\mathbf{x}}_k(t) \in \mathcal{V}, \forall t \in [t_0,t_f], \forall k \in \{1,\ldots,n\}. \tag{8e}\label{eq:MILP_problem_c5}
\end{align}
% \end{equation}
In contrast to existing problem definitions, we consider a continuous-time system %obtaining %a continuous-time temporal robustness 
and maximizing the continuous-time asynchronous temporal robustness which has not been, to the best of our knowledge, performed before. 
% Next, we will see how we obtain an under-approximation of the $\bar{\theta}_{\phi}(\mathbf{x})$ from Def. \ref{def:predATR} that can be applied recursively to a specification $\phi$ from the fragment in Eq. \eqref{eq:stl_fragment}.
% the B\'ezier parametrization enables this yet also how the presented predicate ATR in Sec. \ref{sec:rATR} lets us obtain an under-approximation of $\bar{\theta}$. 
%continuous-time temporal robustness and continuous differentiability. 
% Even though the temporal robustness is continuous, at the heart of our solution is a MILP formulation.

%% file: 4_Parametrization.tex
\section{Multi-Agent B\'ezier parametrization}
\label{sec:bezier_parametrization}
As we are concerned with specifications that are expressed both in the spatial and temporal domains, we wish to use a trajectory formulation that allows us to reason over these dimensions separately. 
To that end, we define a \emph{curvature B\'ezier curve} $r(s)$ and a \emph{temporal B\'ezier curve} $t := h(s)$.
We then couple these curves to parameterize the physical trajectory over time.
We base our multi-agent B\'ezier formulation on the work in \cite{marcucci2022motion} where the authors express a trajectory segment $x$ via the use of $r(s)$ and $h(s)$ as follows:
\begin{gather}
\label{Bezier_formulation}
    r(s) := x(h(s)), \\
    \dot{r}(s) := \dot{x}(h(s))\dot{h}(s). \label{eq:bezier_formulation1}
\end{gather}
It should be noted that $x$ is our physical trajectory of interest, which is parameterized by the two types of B\'ezier curves.

\subsection{Trajectory Construction}
Consider a space of dimension $\mathrm{dim}\cdot n$ (e.g., $n$ agents modeled as the double integrator from Eq. \eqref{eq:double_integrator} moving in a 2D environment). To generate the full trajectory of a single agent $k$, we concatenate $N$ B\'ezier curve (\emph{segments}), denoted as $\mathbf{r}(k) \in \mathcal{B}^{\mathrm{dim}\times N}$ and $\mathbf{h}(k) \in \mathcal{B}^{N}$, where $\mathcal{B}$ indicates the set of B\'ezier curves. An example of the concatenation of curvature B\'ezier curves is shown in Fig.~\ref{fig:beziers}b. %In multi-agent settings, we use $\mathit{n}$ concatentations of curvature B\'ezier curves $\mathbf{r}$ (one for each agent), and a single concatenation of temporal B\'ezier curves $\mathbf{h}$. 

For a multi-agent system consisting of $n$ agents, we consider $n$ concatenations of the aforementioned, leading to the system parametrization by $\mathbf{r}$ and $\mathbf{h}$.
% $\textbf{r} \in \mathcal{B}^{\text{dim}\cdot n \times N}$, $\textbf{h} \in \mathcal{B}^{n \times N}$. In accordance to Eq. \eqref{Bezier_formulation}, let us denote the state $\textbf{x}(k,i)$ for segment $i$ of agent $k$.
Given agents $k \in \{1,\ldots,n\}$ and segments $i \in \{1,\ldots,N\}$, the function $\mathbf{r}$ is defined as $\mathbf{r} : \{1,\ldots,n\} \times \{1,\ldots,N\} \rightarrow \mathcal{B}^{\mathrm{dim}}$ returning a vector of curvature B\'ezier curves. The function $\mathbf{h}$ is defined as $\mathbf{h} : \{1,\ldots,n\} \times \{1,\ldots,N\} \rightarrow \mathcal{B}$ returning a temporal B\'ezier curve. The $b$'th control point of the temporal B\'ezier curve of agent $k$ at segment $i$ is denoted by $\mathbf{h}(k,i)^{(b)}$ and similarly for the curvature, $\textbf{r}(k,i)^{(b)}$, as shown in Fig. \ref{fig:beziers}.

\subsection{Continuity Constraints}
To realize the double integrator dynamics of Eq. \eqref{eq:double_integrator}, we ensure continuous differentiability of $x$ between. We do this in a conservative yet linear manner for each agent $k \in \{1,\ldots,n\}$ by constraining the first and last (i.e. $d$-th for a B\'ezier curve of degree $d$) control points of each segment $i$ according to
\begin{align}
\label{eq:C1_constraints}
    %\begin{aligned}
        & \textbf{r}(k,i)^{(d)} = \textbf{r}(k,i+1)^{(0)} \land 
        \mathbf{h}(k,i)^{(d)} = \mathbf{h}(k,i+1)^{(0)} \nonumber \\ & \Rightarrow \textbf{x}_k(i) = \textbf{x}_k(i+1) \\
    %\end{aligned}\\
    %\begin{aligned}
         & \dot{\textbf{r}}(k,i)^{(d-1)} = \dot{\textbf{r}}(k,i+1)^{(0)} \land 
        \dot{\mathbf{h}}(k,i)^{(d-1)} = \dot{\mathbf{h}}(k,i+1)^{(0)} \nonumber \\ & \Rightarrow \dot{\textbf{x}}_k(i) = \dot{\textbf{x}}_k(i+1).
   % \end{aligned}
\end{align}
We constrain the initial and final time, $t_0$ and $t_f$, and the forward traversal of time via
\begin{align}
    \label{eq:h_constraints}
   & \mathbf{h}(k,0)^{(0)} = t_0, \\
   & \mathbf{h}(k,N)^{(d)} = t_f, \\
   & \dot{\mathbf{h}}(k,i)^{(b)} > 0 \ \forall i \in \{1,\ldots,N\}, \forall b \in \{0,\ldots,d-1\},
\end{align}
for each agent $k \in \{1,\ldots,n\}$. 
The B\'ezier formulation allows us to consider the maximization problem $\argmax_{\mathbf{r},\mathbf{h}}$ as opposed to $\argmax_{\mathbf{u}}$ in Eq. \eqref{eq:MILP_problem_cost} as the acceleration, $\ddot{\mathbf{x}}$, is implicitly parameterized by the curvature and time B\'eziers $\mathbf{r}$ and $\mathbf{h}$ due to the derivative properties of B\'ezier curves as per Sec. \ref{ssec:bezier_curves}. 
% However, not that constraints on the acceleration would be non-convex and we only consider velocity constraints in Eq. \eqref{eq:MILP_problem_cost}.
% \textcolor{red}{mention that now we can take $\argmax_{r,h}$ instead of $\argmax_{u}$}
% This parametrization is suitable for systems with spatio-temporal specifications as it adheres to the properties in Sec. \ref{ssec:bezier_curves} up until the first derivative.

%% file: 5_Predicate_ATR.tex
\section{Recursive ATR on B\'ezier Segments}
\label{sec:rATR}

% \subsection{Recursive ATR on B\'ezier Segments}
% % \textcolor{red}{add why the aforementioned can't be implemented}
% \textcolor{red}{Restrict to positive temporal robustness}

The definition of recursive ATR in Def. \ref{def:predATR} cannot directly be implemented in a motion planning framework, as it is defined over strictly continuous time signals. Instead, we would like to be able to evaluate temporal robustness at a finite set of points. As such, we consider a definition of recursive ATR that considers non-constant time segments from the B\'ezier parametrization in Sec. \ref{sec:bezier_parametrization}.
In the following, in accordance with the constraint in Eq. \eqref{eq:MILP_problem_c2}, we restrict ourselves to strictly positive definitions of temporal robustness. We will show that under this assumption, the recursive ATR on non-singular time segments is sound and an underapproximation of $\bar\theta_{\phi}(\mathbf{x})$. 
% As such, by substituting the ATR in Eq. \eqref{eq:MILP_problem_cost} with this under approximation, we get trajectories that are no worse than 
% We leave a sound definition of negative temporal robustness for future work.

Let $I_{(k,i)} = [\mathbf{h}(k,i)^{(0)},\mathbf{h}(k,i)^{(d)}]$ be the time-span of segment $i$ of agent $k$, see Fig. \ref{fig:toy_example} and assume it is not a single point. Segment Recursive ATR considers allowed segment shifts on the states of individual agents. % in a predicate $p$ that consists of non-constant duration segments.
% Let $\theta_{(k,i)}^{+\pm\pm..}$ indicate the temporal robustness of a predicate $p$ of signal $k$ at segment $i$, where the state of agent $k$ is shifted to the right and the states of agents $l \in \mathcal{S}_p, l\neq k$ are shifted asynchronously. It captures shifts $\kappa_k \in [0,\tau]$, $\kappa_l \in [-\tau,\tau]$. 
% Lastly, note that the nonlinear nature of the B\'ezier curve in Eq. \eqref{eq:bezier_formulation1} allow us to only make statements about the end-points of the B\'ezier curve in linear fashion. 
\begin{definition}[Segment Recursive ATR]
\label{def:predATR_segment}
    Consider the time-shifted state $\mathbf{x}_{p,\bar{\kappa}}(t)$ from Eq. \eqref{eq:kappa_bar}.
    %We consider the state of each agent in predicate $p$, and define 
    First, for a predicate $p$, Segment Recursive ATR for an agent $k \in \mathcal{S}_p$ and a segment $i \in \{1,\ldots,N\}$ defines the maximum extent the state of each agent $l \in \mathcal{S}_p$, including $k$, may shift asynchronously such that all segments spanning the original time-span of segment $i$, $I_{(k,i)}$, satisfy $\chi_p(k,i)$. 
    % the predicate holds for the entire segment $i$. 
    % Additionally, we must consider index $i$ being replaced by other indices as sub-signal $k$ itself is shifted by $\kappa_k \in [-\tau,\tau]$.
    % indicating that the stacked shifted state is a state vector for which each signal $x_k$ is potentially shifted by a unique delay ($\kappa_k < 0$) or advancement ($\kappa_k > 0$).
    % We define the ATR of predicate $p$ for segment $i$ of signal $k$ as
    \begin{multline}
    \label{eq:ATR^p_segment}
        \theta_{p}(\mathbf{x},k,i) :=  \chi_{p}(\mathbf{x}_p,k,i) \cdot \\
        \max\left\{\tau \ge 0 : \begin{array}{c} \bar{\kappa} \in [-\tau,\tau], \forall l \in \mathcal{S}_p, \\ \chi_{p}(\mathbf{x}_{p,\bar{\kappa}},k,j) = \chi_{p}(\mathbf{x}_p,k,i), \forall j \in J(\tau) \end{array} \right\},
    \end{multline}
    where $\chi_{p}(\mathbf{x}_p,k,i)$ is the characteristic function for a segment $i$ of agent $k$,
    \begin{equation}
    \label{eq:chi^p}
        \chi_{p}(\mathbf{x}_p,k,i) := \begin{cases}
                1,  &\text{if} \enspace \mu(\mathbf{x}_p,t) \geq 0 \quad \forall t \in I_{(k,i)} \\
                % , \forall i : I_{(k,j)} \cap I_{(l,i)} \neq \varnothing, \\ \forall l \in \mathcal{S}_p, l \neq k 
            -1, &\text{otherwise,}
        \end{cases}
    \end{equation}
    which relies on evaluating the predicate for all control points of segment $i$.
    The set $J(\tau)$ is dependent on $\tau$ and defined as
    \begin{multline}
    \label{eq:J}
        J(\tau) := \big\{j \in \{1,\ldots,N\} \big\rvert \\
        \exists \kappa_k \in [-\tau,\tau] : I_{(k,j)} + \kappa_k \cap I_{(k,i)} \neq \varnothing \big\}.
    \end{multline}
    % Consideration of $J$ ensures that $\forall t \in I_{(k,i)}, \theta^{p}_{(k,i)}(\mathbf{x})\big\rvert_{t} \leq \theta^{p}_{(k,i)}(\mathbf{x})$.
    % Based on this definition of the ATR per segment of a signal $k$ on a predicate $p$, 
    We additionally define Segment Recursive ATR for the temporal operators of the STL fragment in Eq. \eqref{eq:stl_fragment} as
    % As such, for the temporal robustness of signal $k$, we need to consider the temporal robustness when signal $k$ shifts to the left and to the right separately, (e.g. $\theta^{+\pm\pm..}_{(k,i)}$ captures $\kappa_k \in [0,\tau], \kappa_l \in [-\tau,\tau]$ and $\theta^{-\pm\pm..}_{(k,i)}$ captures $\kappa_k \in [-\tau,0], \kappa_l \in [-\tau,\tau]$) and add or subtract the conservativeness that is introduced due to the segments. Let the ATR for the temporal operators for signal $k$ and for the system be obtained as
    \begin{equation}
    \label{eq:segment_always_k}
        \theta_{\Box_I p}(\mathbf{x},k) := \min_{i\in \{1,...,N\},I_{(k,i)}\cap I \neq \varnothing}(\theta_{p}(\mathbf{x},k,i)),
    \end{equation}
    \begin{equation}
    \label{eq:segment_always}
        \theta_{\Box_I p}(\mathbf{x}) := \min_{k\in \mathcal{S}_p}(\theta_{\Box_I p}(\mathbf{x},k)),
    \end{equation}
    \begin{equation}
    \label{eq:segment_eventually_k}
        \theta_{\diamondsuit_I p}(\mathbf{x},k) := \max_{i \in \{1,...,N\},I_{(k,i)}\cap I \neq \varnothing}(\theta_{p}(\mathbf{x},k,i)),
    \end{equation}
    \begin{equation}
    \label{eq:segment_eventually}
        \theta_{\diamondsuit_I p}(\mathbf{x}) := \min_{k\in \mathcal{S}_p}(\theta_{\diamondsuit_I p}(\mathbf{x},k)),
    \end{equation}
    % \begin{multline}
    % \label{eq:ct_theta_temporal_def}
    %     \theta^{\Box_I p}_{(k)}(\mathbf{x}) := \min \big( \min_{i, I_{(k,i)}\cap I \neq \varnothing}(\theta^{+\pm\pm...}_{(k,i)} + (h^{(d)}_{(k,i)} - \bar{I})), \\
    %     \min_{i, I_{(k,i)}\cap I \neq \varnothing}(\theta^{-\pm\pm...}_{(k,i)} + (\underline{I} - h^{(0)}_{(k,i)})),\ldots \big),
    % \end{multline}
    % \begin{equation}
    %     \theta^{\Box_I p}(\mathbf{x}) = \min_{k\in \mathcal{S}_p}(\theta^{\Box_I p}_{(k)}(\mathbf{x})).
    % \end{equation}
    % \begin{multline}
    %     \theta^{\diamondsuit_I p}_{(k)}(\mathbf{x}) := \min \big( \max_{i, I_{(k,i)}\cap I \neq \varnothing}(\theta^{+\pm\pm...}_{(k,i)} - (h^{(d)}_{(k,i)} - \bar{I})), \\
    %     \max_{i, I_{(k,i)}\cap I \neq \varnothing}(\theta^{-\pm\pm...}_{(k,i)} - (\underline{I} - h^{(0)}_{(k,i)})),\ldots \big),
    % \end{multline}
    % \begin{equation}
    %     \theta^{\diamondsuit_I p}(\mathbf{x}) = \min_{k\in \mathcal{S}_p}(\theta^{\diamondsuit_I p}_{(k)}(\mathbf{x})).
    % \end{equation}
    and for the Boolean operators as
    \begin{equation}
    \label{eq:segment_binary}
    \begin{aligned}
        \theta_{\phi_1 \land \phi_2}(\mathbf{x}) &:= \min(\theta_{\phi_1}(\mathbf{x}),\theta_{\phi_2}(\mathbf{x})), \\
        \theta_{\phi_1 \lor \phi_2}(\mathbf{x}) &:= \max(\theta_{\phi_1}(\mathbf{x}),\theta_{\phi_2}(\mathbf{x})),
    \end{aligned}
    \end{equation}
    obtaining the ATR of specification $\phi$ as $\theta_{\phi}(\mathbf{x})$.
    \end{definition}
% Note that determining the temporal robustness of $\Box$ and $\diamondsuit$ via the $\min$ and $\max$ operators in Eq. \eqref{eq:segment_always_k} and \eqref{eq:segment_eventually_k} is an underapproximation as segment $i$ of agent $k$ might only partially overlap $I$ (see Fig. \ref{fig:toy_example}). 
% Notice that for $\theta^{\Box_I p}(\mathbf{x})$ and $\theta^{\diamondsuit_I p}(\mathbf{x})$ an intersecting B\'ezier curve, $I_{(k,i)} \cap I = I_{(k,i)}$, indicates that the whole segment $I_{(k,i)}$ should intercept $I$. This introduces conservativeness in the solution. We will see in the B\'ezier encoding in Sec. \ref{sec:MILP_encoding} how this is remedied.
The set $J$ indicates all segments of $k$ that now intersect or have intersected the original time span of segment $i$ under the time-shift $\kappa_k \in [-\tau,\tau]$, which includes $i$ itself.

\begin{figure}[t!]
    \centering
    \includegraphics[width=0.4\textwidth]{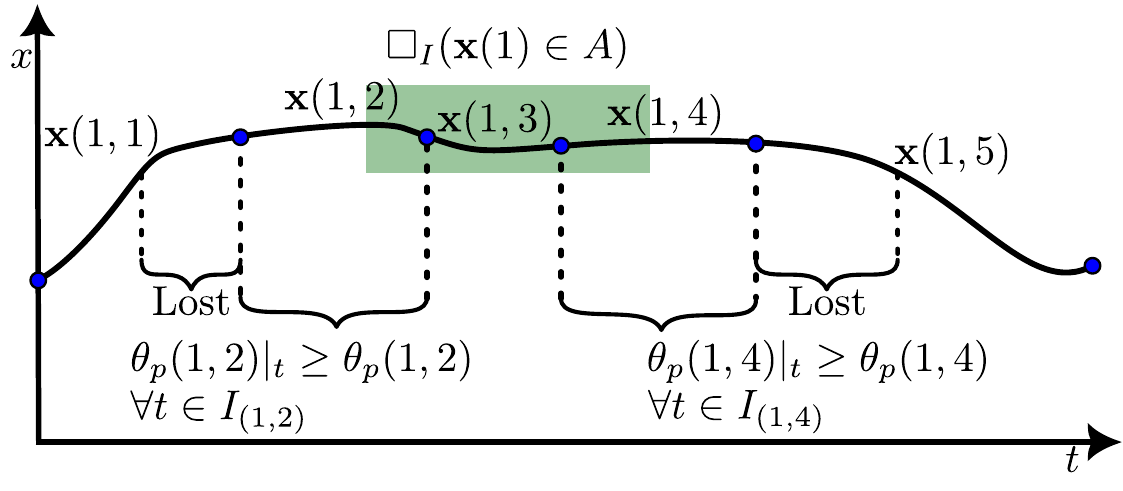}
    \caption{Example of Always staying within a certain range $x \in A$ for a certain time interval $I$. 
    Notice the under-approximation of the true temporal robustness of Def. \ref{def:predATR} as the $\theta_p(k,i)$ is assigned the lower-bound of $\theta_p(k,i)\rvert_t, \forall t \in I_{(k,i)}$ and the end-point evaluation of the B\'ezier curve (Lost).}
    \label{fig:toy_example}
    \vspace{-0.5cm}
\end{figure}

% We are now able to state our main theorem
\begin{theorem} For positive ATR, $\theta_{\phi}(\mathbf{x}) \leq \bar\theta_{\phi}(\mathbf{x})$, indicating an underapproximation of the ATR from Def. \ref{def:predATR}.
% On the fragment from Eq. \eqref{eq:stl_fragment}, the ATR of a specification $\phi$, $\theta^{\phi}(\mathbf{x})$, recursively obtained from $\theta^{p}_{(k,i)}(\mathbf{x})$ from Def. \ref{def:p_ATR} and the subsequent temporal and binary operators from Eq. \eqref{eq:ct_theta_temporal_def} - \eqref{eq:ct_theta_binary_def}, is upper-bounded by the specification ATR $\bar{\theta}^{\phi}(\mathbf{x})$ from \cite{lindemann2022temporal} $\forall t \in [t_0,t_f]$, meaning the definitions provided are sound.
\end{theorem}
\begin{proof}
First, we show that $\theta_p(\mathbf{x},k,i) \leq \bar\theta_p(\mathbf{x},t)$, for all $t \in I_{(k,i)}.$ 
    %Def. \ref{def:p_ATR} assigns temporal robustness to a segment $i$ for agent $k$ makes statements on a B\'ezier segment that is valid $\forall t \in I_{(k,i)}$ and as such $\theta_p(\mathbf{x},k,i)\rvert_t \leq \theta_p(\mathbf{x},k,i), \forall t \in I_{(k,i)}$. 
    From Eq. \eqref{eq:chi^p} it follows that $\chi_p(\mathbf{x},k,i) \leq  \bar{\chi}_p(\mathbf{x},t)$, for all $t \in I_{(k,i)}$.
    Note that we assume that $\chi_p(\mathbf{x},k,i) = 1$. 
    From Eq. \eqref{eq:chi^p} and Eq. \eqref{eq:J}, it  follows that for all $\bar{\kappa}$, if $\chi_p(\mathbf{x}_{p,\bar{\kappa}},k,j) = \chi_p(\mathbf{x},k,i) =1, \forall j \in J(\tau)$, then also $\bar{\chi}_p(\mathbf{x}_{\bar{p,\kappa}},t) = \bar{\chi}_p(\mathbf{x},t) =1 , \forall t \in I_{(k,i)}$.
    Hence, $\chi_p(\mathbf{x}_{p,\bar\kappa},k,j) \leq  \bar{\chi}_p(\mathbf{x}_{p,\bar\kappa},t)$, $\forall j \in J, t \in I_{(k,i)}$.  Hence, $\tau^\star \leq \bar\tau^\star$, where $\tau^\star$ is the maximizing $\tau$ in Eq. \ref{eq:ATR^p_segment} and similarly $\bar \tau^\star$ is the maximizing $\bar \tau$ in Eq. \ref{eq:ATR^p}. Considering the positive value of the characteristic function, altogether we have $\theta_p(\mathbf{x},k,i) \leq \bar\theta_p(\mathbf{x},t)$.

    Next, consider $\theta_{\Box_I p}(\mathbf{x},k)$.
    As we have established that $\theta_p(\mathbf{x},k,i) \leq \bar\theta_p(\mathbf{x},t), \forall t \in I_{(k,i)}$, we can state that $\theta_p(\mathbf{x},k,i) \leq \bar\theta_p(\mathbf{x},t), \forall t \in (I \cap I_{(k,i)})$. 
    Hence, consider that by taking $\min$ over all segments $i$ for which $I_{(k,i)} \cap I \neq \varnothing$ we obtain a temporal robustness $\theta_{\Box_I p}(\mathbf{x},k) \leq \bar\theta_{\Box_I p}(\mathbf{x},t)$.
    To obtain the overall ATR, $\theta_{\Box_I}$, we need to consider each agent $k \in \mathcal{S}_p$ as described in Eq. \eqref{eq:segment_always}.
    From Def. \ref{def:predATR_segment} it follows that $\bar\theta_{\Box_I p}(\mathbf{x},k) = \bar\theta_{\Box_I p}(\mathbf{x},l), \forall k,l \in \mathcal{S}_p$ and hence by taking the $\min$ over all agents we obtain $\theta_{\Box_I p}(\mathbf{x}) \leq \bar\theta_{\Box_I p}(\mathbf{x})$.

    Next, consider $\theta_{\diamondsuit_I p}(\mathbf{x},k)$.
    Consider that by taking $\max$ over all segments $i$ for which $I_{(k,i)} \cap I \neq \varnothing$ we obtain a temporal robustness $\theta_{\diamondsuit_I p}(\mathbf{x},k) \leq \bar\theta_{\diamondsuit_I p}(\mathbf{x})$.
    Similar to $\Box_I p$, we consider each agent in $\mathcal{S}_p$ and obtain $\theta_{\diamondsuit_I p}(\mathbf{x}) \leq \bar\theta_{\diamondsuit_I p}(\mathbf{x})$.
    
    The binary $\land$ and $\lor$ are trivial. Given the considered fragment of Eq. \eqref{eq:stl_fragment}, this concludes the proof.
\end{proof}
\begin{remark}
    As $N\rightarrow \infty$, Segment Recursive ATR collapses to Recursive ATR in Def. \ref{def:predATR} as $\theta_{\Box_I p}(\mathbf{x},k) = \theta_{\Box_I p}(\mathbf{x},l),\forall k,l\in \mathcal{S}_p$ and $\theta_p(\mathbf{x},k,i) \rightarrow \bar\theta_p(\mathbf{x},t)$ as $I_{(k,i)} \rightarrow t$.
\end{remark}

We redefine the problem in Sec. \ref{sec:problem_statement} by maximizing the segment recursive ATR, $\theta_{\phi}(\mathbf{x})$, as opposed to $\bar{\theta}_{\phi}(\mathbf{x})$ in Eq. \eqref{eq:MILP_problem_cost} and constraining satisfaction by $\theta_{\phi}(\mathbf{x}) > 0$ as opposed to $\bar{\theta}_{\phi}(\mathbf{x}) > 0$ in Eq. \eqref{eq:MILP_problem_c2}. 

%% file: 6_MILP_encoding.tex
\section{MILP encoding}
\label{sec:MILP_encoding}
In this section, we will address the computation of B\'ezier parametrization that can handle general motion planning constraints, STL specifications, and their temporal robustness. We will specifically elaborate on how Segment Recursive ATR from Def. \ref{def:predATR_segment} is implemented in a MILP.

\subsection{Trajectory Constraints}
Let us first define the trajectory constraints from Eq. \eqref{eq:MILP_problem_c1}, \eqref{eq:MILP_problem_c3}, \eqref{eq:MILP_problem_c4}, and \eqref{eq:MILP_problem_c5} with the B\'ezier parametrization from Sec. \ref{sec:bezier_parametrization}. First, we constrain the initial and final state of the trajectory (corresponding to Eq. \eqref{eq:MILP_problem_c3}) by constraining the first and last control point of the position and velocity of the first and last segment respectively, according to
\begin{align}
    \mathbf{x}_k(t_0) = \mathbf{x}_{k,t_0} & \iff \mathbf{r}(k,0)^{(0)} = \mathbf{x}_{k,t_0} \\
    \dot{\mathbf{x}}_k(t_0) = \dot{\mathbf{x}}_{k,t_0} & \iff \mathbf{\dot{r}}(k,0)^{(0)} = \dot{\mathbf{h}}(k,0)^{(0)} \dot{\mathbf{x}}_{k,t_0}, \\
    \mathbf{x}_k(t_f) = \mathbf{x}_{k,t_f} & \iff \mathbf{r}(k,N)^{(d)} = \mathbf{\mathbf{x}}_{k,t_f}, \\
    \dot{\mathbf{x}}_k(t_f) = \dot{\mathbf{x}}_{k,t_f} & \iff \dot{\mathbf{r}}(k,N)^{(d-1)} = \dot{\mathbf{h}}(k,N)^{(d-1)}\dot{\mathbf{x}}_{k,t_f}.
\end{align}
%which allows us to constrain the initial and final state of the trajectory.
We then address constraint Eq. \eqref{eq:MILP_problem_c4} by introducing a constraint that each control point of segment $i$ of agent $k$ lies within the convex polygon workspace and that all control points of segment $i$ of agent $k$ lay outside at least one face of every convex obstacle polygon
\begin{multline}
    \label{eq:constraint:w_free}
        \mathbf{x}_k(i) \in \mathcal{W}_{\mathit{free}} \impliedby \bigwedge_{f=1}^{n_w^{\mathit{faces}}} \bigwedge_{j=0}^{d} (b_{w}^{f} - H_{w}^{f} \mathbf{r}(k,i)^{(j)} \geq 0) \land \\ 
        \bigwedge_{l=1}^{n^{\mathit{obs}}} \bigvee_{f=1}^{n_l^{\mathit{faces}}} \bigwedge_{j=0}^d (b^f_l - H^f_l \mathbf{r}(k,i)^{(j)} \leq 0 ),
\end{multline}
where $n_{\mathit{obs}}$ defines the number of obstacles, $n_w^{\mathit{faces}}$ defines the number of faces of the world polygon, $\mathcal{W}_{free}$, and $n_l^{\mathit{faces}}$ defines the number of faces of obstacle polygon $l$, $H_w^f$ and $b_w^f$ are the $f$'th row of the matrix and vector for the half-space inequality conditions of the convex {\em world} polygon. $H_l^kf$ and $b_l^f$ denote the equivalent for obstacle $l$.    
An example of this is shown in Fig. \ref{fig:beziers}b.

Lastly, we constrain the velocity (constraint in Eq. \eqref{eq:MILP_problem_c5})
\begin{multline}
    \label{eq:constraint_dotq}
        \dot{\mathbf{x}}_k(i) \in \mathcal{V} \impliedby \dot{\mathbf{h}}(k,i)^{(j)}\underline{\dot{\mathbf{x}}}_k \leq \dot{\mathbf{r}}(k,i)^{(j)} \leq \dot{\mathbf{h}}(k,i)^{(j)}\bar{\dot{\mathbf{x}}}_k, \\
        \forall j \in \{0,\ldots,d-1\}.
\end{multline}
Notice the workspace and velocity constraints, \eqref{eq:constraint:w_free} and \eqref{eq:constraint_dotq}, rely on the convex-hull overapproximation and bounding box and can be conservative.

\subsection{STL Constraints}
We use $\theta^{\pm\ldots\pm}_p(\mathbf{x},k,i)$, where $\pm \in \{+,-\}$, to denote the temporal robustness of predicate $p$ for agent $k$ at segment $i$ where multiple agents influence the truth value of the predicate. For example, if $S_p =\{k,l\}$, $\theta^{+-}_p(\mathbf{x},k,i)$ indicates that signal $k$ is shifted to the right and signal $l, l\neq k$ is shifted to the left. It hence captures the shifts $\forall \kappa_k \in [0,\tau], \kappa_l \in [-\tau,0]$. Similarly $\theta^{++}_p(\mathbf{x},k,i)$ captures the shifts $\forall \kappa_k, \kappa_l \in [0,\tau]$, etc. $\theta^{\pm\pm}_p(\mathbf{x},k,i)$ is then the overall robustness of $p$, regardless of the direction of the time shift, equivalent to Eq. \eqref{eq:ATR^p_segment}. We also use $\theta^{0\pm\ldots\pm}_p(\mathbf{x},k,i)$, where $\pm \in \{+,-\}$ to indicate the temporal robustness when agent $k$ is not subject to a time shift, while the other agents $l \in S_p, l \neq k$ are.
%We use indicator variables $z \in \mathbb{B}$.

\subsubsection{Predicates}
Let us first define the satisfaction of a predicate for a B\'ezier segment $i$ of agent $k$. %to a predicate. 
A B\'ezier curve can only be evaluated in a linear fashion at the first and last control point. In order to keep computations feasible, we define satisfaction of a predicate $p$ for the segment $i$ via the bounding box of segment $i$. For stay-in polygon predicate $p$, this would result in the indicator variable $z^{\mathit{spat}}_{(k,i)} \in \mathbb{B}$ for satisfaction of predicate $p$ for the segment $i$ of agent $k$ according to
%We can only make statements about the first and last control point and 
%We hence define the satisfaction of a predicate for segment $i$ via a constraint that requires to contain segment $i$ in a particular polygon:
\begin{equation}
\label{eq:z^spat}
    z^{\mathit{spat}}_{(k,i)} = \bigwedge_{f=1}^{n_{\mathit{faces}}} \bigwedge_{j=0}^d (b^{f}-H^{f} \mathbf{r}(k,i)^{(j)} \geq 0).
\end{equation}
We can additionally consider interdependent predicates, for example, $\mu = (|\mathbf{x}_1 - \mathbf{x}_2|_{\infty} \leq \epsilon)$ from Fig. \ref{fig:intro}.
% specifying that robot 1 and robot 2 should stay $\epsilon$-close to each other.
Given that each agent has its own temporal curve per segment, $\mathbf{h}(k,i): \{1,\ldots,n\} \times \{1,\ldots,N\} \rightarrow \mathcal{B}$, we can only look at predicate satisfaction by first considering the intersections of B\'ezier segments between agents via
\begin{equation}
    z^{\mathit{temp}}_{(k,i),(l,j)} := \{I_{(k,i)}\cap I_{(l,j)} \neq \varnothing\},
\end{equation}
where $z^{\mathit{temp}}_{(k,i),(l,j)} \in \mathbb{B}$ is a variable indicating whether segment $i$ of agent $k$ and segment $j$ of agent $l$ intersect.
We can obtain the equivalent of Eq. \eqref{eq:z^spat} for predicates that depend on the states of multiple agents 
\begin{equation}
    z^{\mathit{spat}}_{(k,i)} := \bigwedge z^{\mathit{spat}}_{(k,i),(l,j)} \quad \forall j : z^{\mathit{temp}}_{(k,i),(l,j)} = \top, \forall l \in \mathcal{S}_p, l \neq k,
\end{equation}
that states that $z^{\mathit{spat}}_{(k,i)} \in \mathbb{B}$ is true if and only if all segments of agent $l$ that intersect segment $i$ of agent $k$ satisfy the predicate function. 
For a predicate $p$ that is dependent on 3 agents, this would result in a tensor of indicator variables $z^{\mathit{spat}} \in \mathbb{B}^{N \times N \times N}$. 
\\\\
Implementing the segment recursive ATR from Def. \ref{def:predATR_segment} for agent $k$ and segment $i$ relies on several steps: 
%The predicate ATR , is obtained as follows:

\begin{figure}[t!]
    \centering
    \includegraphics[width=0.4\textwidth]{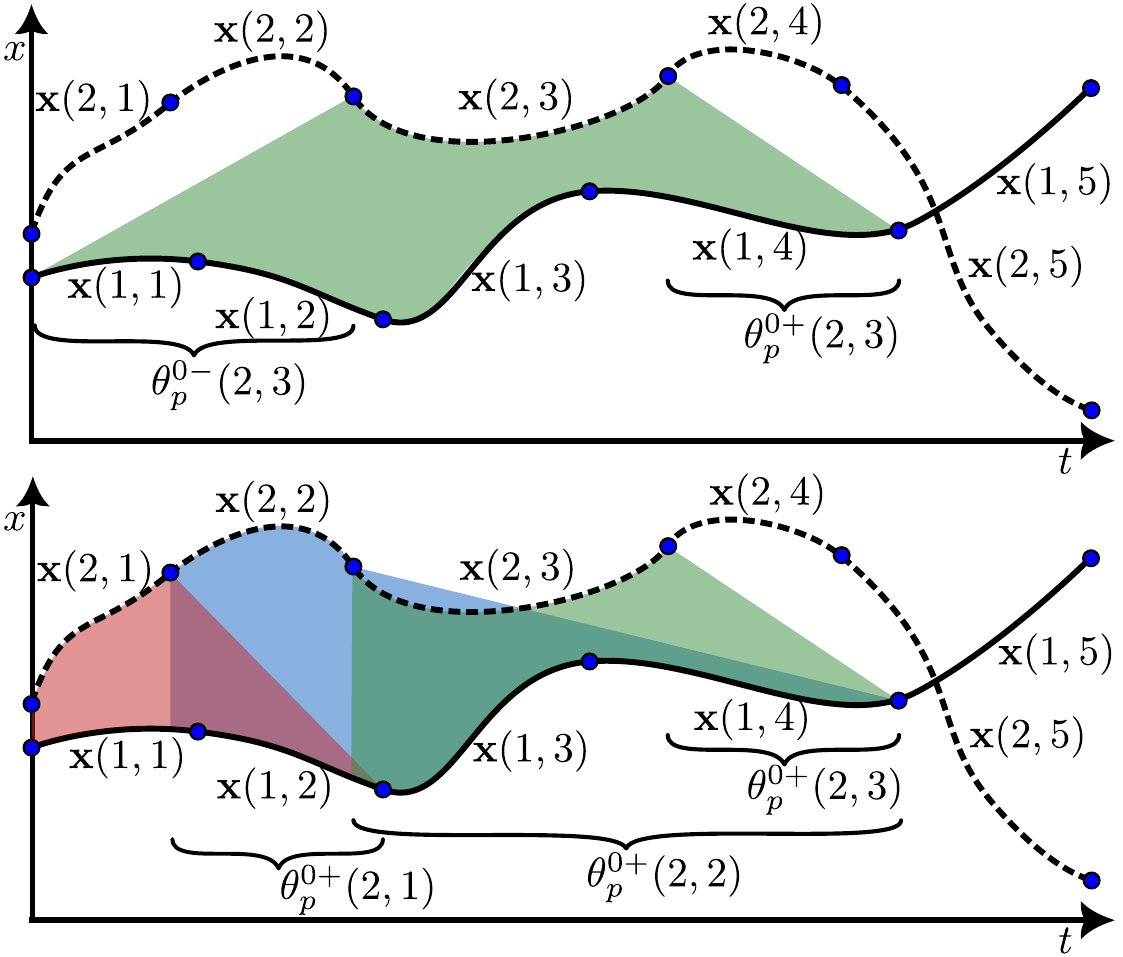}
    \caption{Top: Right and left temporal robustness of the third index of agent $2$ with respect to shifts in the state of agent $1$ for predicate $p = \mathbf{x}_{1} \leq \mathbf{x}_{2}$. Bottom: Right temporal robustness of indices of agent $2$ w.r.t. agent $1$, where the state of agent $2$ is static. we need to consider these when we look at the ATR of segment $3$ of agent $2$.}
    \label{fig:theta_pm}
    \vspace{-0.5cm}
\end{figure}

\begin{enumerate}
    \item First, for each agent $l \in \mathcal{S}_p, l \neq k$, we create two indicator arrays. $z^{+}_{(l,\cdot)(k,i)} \in \mathbb{B}^N$ indicates the segments of $l$ that appear after intersecting $I_{(k,i)}$, $z^{-}_{(l,\cdot)(k,i)} \in \mathbb{B}^N$ indicates the segments of agent $l$ that appear before intersecting
    \begin{gather}
        z^{+}_{(l,\cdot)(k,i)} = \{j \mid \mathbf{h}(l,j)^{(d)} \geq \mathbf{h}(k,i)^{(d)}\}, \\
        z^{-}_{(l,\cdot)(k,i)} = \{j \mid \mathbf{h}(l,j)^{(0)} \leq \mathbf{h}(k,i)^{(0)}\}.
    \end{gather}
    
    \smallskip
    For the example illustrated in Fig. \ref{fig:theta_pm}, $$z^{+}_{(1,\cdot),(2,2)} = \begin{bmatrix} 0 & 0 & 0 & 1 & 1 \end{bmatrix}$$  $$z^{-}_{(1,\cdot),(2,2)}=\begin{bmatrix} 1 & 1 & 0 & 0 & 0 \end{bmatrix}$$ with respect to segment $I_{(k,i)}=I_{(2,2)}$.

    \medskip
    
    \item As we need to consider the temporal robustness of the segments $j \in J(\tau)$ per Eq. \eqref{def:predATR_segment}, we first obtain the temporal robustness for segment $i$ of agent $k$ while the state of agent $k$ is static for all combinatorial time shifts of agent $l \in \mathcal{S}_p, l\neq k$, e.g., $\theta_p^{0++..}(k,i)$, $\theta_p^{0+-..}(k,i)$, $\theta_p^{0-+..}(k,i)$, $\theta_p^{0--..}(k,i)$, etc.
    We obtain $\theta_p^{0\pm\pm..}(k,i)$ by considering the minimum duration $\forall l \in \mathcal{S}_p$ when $\chi^p_{(l,j)} \neq \chi^p_{(k,i)}$ occurs for the first time and $z^{+}_{(l,),(k,i)}=1$ or $z^{-}_{(l,),(k,i)}=1$ depending on the direction of the shift under consideration.
    For obtaining $\theta_p^{0+}(k,i)$, we consider an $N\times 1$ column of $z^{\mathit{spat}} \in \mathbb{B}^{N \times N}$.

    \smallskip
    
    For the example in Fig. \ref{fig:theta_pm}, agent $2$ and segment $2$, we obtain $$z^{\mathit{spat}}_{(2,2)} = \begin{bmatrix} 1 & 1 & 1 & 1 & 0 \end{bmatrix}$$ for which $\theta_p^{0+}(2,2) = \mathbf{h}(1,4)^{(0)} - \mathbf{h}(2,2)^{(d)}$, i.e. the first occurrence of $\chi_p(l,j) \neq \chi_p(k,i)$. 
    Equivalently, $\theta_p^{0-}(2,2) = \mathbf{h}(2,2)^{(0)} - \mathbf{h}(1,0)^{(0)}$.
    \medskip
    
    \item Then, realize that according to Def. \ref{def:predATR_segment}, the state of agent $k$ shifts and as such, element $i$ gets replaced by $j \in J$ in accordance to Eq. \eqref{eq:J}. As such, for all combinatorial shifts in $\theta_p^{0\pm\pm..}(k,i)$, for $\pm \in \{+,-\}$, we consider the forward and backward (or right and left) shift of the state of agent $k$. 
    This results in $\theta_p^{+\pm\pm..}(k,i)$ and $\theta_p^{-\pm\pm..}(k,i)$.
    We construct $\theta_p^{+\pm\pm..}(k,i)$ by adding the segment durations $I_{(k,j)}, j \in J(\tau), j\leq i$ to $\theta_p^{+\pm\pm..}(k,i)$ only if there is sufficient temporal robustness of $\theta_p^{0\pm\pm..}(k,j)$ for the predicate to hold when $I_{(k,j)} \cap I_{(k,i)} \neq \varnothing$. 

    \smallskip
    An example of this is indicated in Fig. \ref{fig:theta_pm} in the bottom for a predicate dependent on the state of two agents.

    \medskip

    \item At last, we obtain the Segment Recursive ATR of segment $i$ of agent $k$ as
        \begin{equation}
            \theta_p(\mathbf{x},k,i) := \min(\theta_p^{+\pm\pm..}(\mathbf{x},k,i),\theta_p^{-\pm\pm..}(\mathbf{x},k,i)).
        \end{equation}
  %  that represents the predicate ATR from Def. \ref{def:predATR_segment}.
\end{enumerate}
The above procedure is followed for each agent $l \in \mathcal{S}_p$ for each segment $i \in \{1,\ldots,N\}$. We cannot translate the recursive ATR $\theta_p(k,i)$ to any of the segments of agent $l, l\neq k$ due to the non-constant duration of the B\'ezier segments.

\medskip

\subsubsection{Temporal and Binary Operators}
% We now define a binary variable $z^{\text{temp}}_{k,i/I} \in \mathbb{B}$ that indicates whether B\'ezier segment $i$ of sub-signal $k$ relevant to predicate $p$ intersects the time-interval $I$ of the temporal operator
% \begin{equation}
% \label{eq:parse_time}
%     z^{\text{temp}}_{k,i/I} := \{I_{k,i}\cap I \neq \varnothing\}.
% \end{equation}
The temporal and binary operators are directly applied via Eq. \eqref{eq:segment_always} - \eqref{eq:segment_binary}. It should be noted that these definitions mean that while our temporal robustness is continuous, it is evaluated at a discrete number of points in space. We obtain $\theta_{\phi}(\mathbf{x})$ and solve Eq. \eqref{eq:MILP_problem_cost} with the underapproximation of $\theta_{\phi}(\mathbf{x}) \leq \bar\theta_{\phi}(\mathbf{x})$ for the cost and the constraint \eqref{eq:MILP_problem_c2}.

\subsection{Computational Complexity}
% The complexity of Eq. \eqref{eq:MILP_problem_cost} can be approximated via the number of binary variables, given that we have formulated it as a Mixed-Integer Linear Program (MILP). 
% Given that we have a limited fragment consisting of conjunctions and disjunctions of the temporal operators {\em Always} and {\em Eventually}, we can let us first define the number of variables needed to compute the ATR for an index $i$ of signal $k$. Then we consider the number of binary variables for the temporal operators.
For a multi-agent predicate, $p$, we approximate its complexity in the optimization problem by considering the binary variables needed to express its ATR. The complexity is
\begin{equation}
    % \mathcal{O}(12\cdot N^{|\mathcal{S}_{p}|}) = \mathcal{O}(N^{|\mathcal{S}_{p}|})
    \mathcal{O}(N^{|\mathcal{S}_{p}|})
\end{equation}
where $N$ is the number of segments and $|\mathcal{S}_p|$ indicates the number of agents affecting the truth value of predicate $p$. 
% The multiplication with $12$ is due to the MILP encoding in Sec. \ref{sec:MILP_encoding}. 

Though the complexity is exponential, the B\'ezier formulation with non-constant segment durations allows for a relatively low number of segments $N$, regardless of long-horizon missions, as the duration of a single temporal curve could be large. This will become apparent in the next section.

%% file: 7_Results.tex
\section{Results}
\label{sec:results}
% \textcolor{red}{TODO: this is still very minimal, but I think a separate comparison between existing temporal robustness and the timed-waypoint methods would be interesting. Then perhaps some fancy B\'ezier methods later.}
We will show how our B\'ezier MILP implementation of temporal robust motion planning is computationally faster compared to constant discretization MILP implementations in \cite{rodionova2021time} while obtaining a continuous-time underapproximation of the theoretically optimal temporal robustness.
% Our B\'ezier MILP implementation of temporal robustness has promised a speed-up while providing a continuous time underapproximation of the theoretically optimal temporal robustness according to Def. \ref{def:predATR_segment}.
Additionally, we show how we are able to directly optimize the ATR of a specification. %, albeit to a limited fragment of STL
We will compare our method to a scenario in \cite{rodionova2021time} to highlight the theoretical maximum temporal robustness and conclude by highlighting the optimization capabilities of our approach.
% and compare our method to scenarios of ATR constraints in \cite{yu2023efficient}.
%We begin by comparing our method to the method in \cite{rodionova2021time}. 
Note that compared to \cite{rodionova2021time}, we are unable to handle acceleration constraints and do not allow nesting of temporal operators. However, we are able to compute continuous-time plans. 
We implement the left- and/or right-temporal robustness via the counting procedure described in \cite{rodionova2021time}, but consider the non-constant B\'ezier segment duration.

All simulations are performed on an Intel Core 12700H processor with 32 Gb of RAM. The method is implemented in Python and solved with Gurobi 10.0 \cite{gurobi2021gurobi}.

\begin{figure}[t!]
    \centering
    \includegraphics[width=0.5\textwidth]{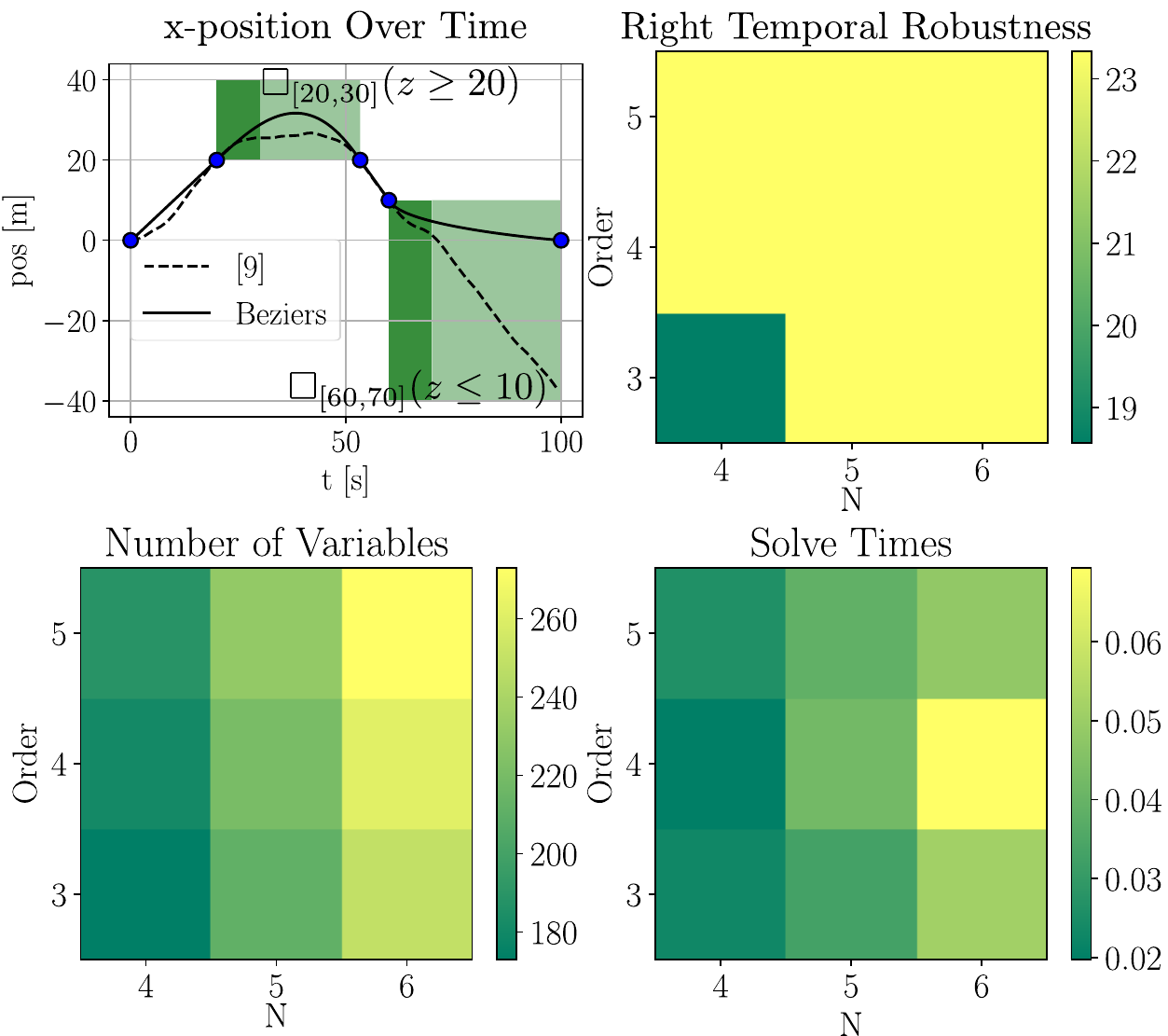}
    \caption{Single-agent STL specification $\phi_{uav}$ for $N=d=4$. Darker green indicates the predicate (and its duration), and lighter shades indicates the temporal robustness. Heatmaps indicate the dependency on the order and the number of B\'ezier curves in finding the optimal solution.}
    \label{fig:UAV}
    \vspace{-0.5cm}
\end{figure}

\subsection{Single-agent Case}
Consider the altitude control of a one-dimensional UAV moving along the $z$-axis with state $x =[z,\dot{z}]^T \in \mathbb{R}^2$, initial state $x_0 = [0,0]^T$, linear second-order dynamics, velocity bounds $|\dot{z}|_{\infty} \leq 1.5$, and the specification $\phi_{uav} = \Box_{[20,30]}(z\geq 20) \land \Box_{[60,70]}(z\leq 10)$.
We remove the control input constraints from the benchmark.
The simulation results are shown in Fig. \ref{fig:UAV}.

The MILP formulation in \cite{rodionova2021time} requires $100$ $1$-second segments and obtains a maximum right temporal robustness of $23$ seconds, limited by the discretization of 1-second intervals. Meanwhile, our B\'ezier formulation with only $4$ segments obtains the theoretically maximal right temporal robustness of $23.33$ seconds. Over 10 trials, \cite{rodionova2021time} takes an average of 0.646 seconds while our B\'ezier method takes 0.019 seconds to compute.

\subsection{Multi-agent Scenario}
To assess the capabilities of maximizing the ATR of a multi-agent specification, we present the scenario from Fig.~\ref{fig:intro} and maximize its ATR. We describe the agents as double integrators subjected to the specification
% introduce and compare our method against the efficient ATR constraint encoding from \cite{yu2023efficient}. Afterward, we present a complex multi-agent scenario.
% \subsubsection{Multi-Agent ATR}
% Consider Fig. \ref{fig:intro}, which describes a two-agent discrete-time double integrator specification of the form
\begin{multline}
\label{eq:phi_MA}
    \phi_{MA} = \diamondsuit_{[4,8]}(\mathbf{x}_1 \in B) \land \Box_{[4,8]}(\mathbf{x}_2 \in A) \land  \\ \Box_{[12,15]}(|\mathbf{x}_1 - \mathbf{x}_2|_{\infty} \leq 1),
\end{multline}
where $\mathbf{x}_1$ and $\mathbf{x}_2 \in \mathbb{R}^2$ represent the position of robots 1 and 2 respectively. 
We consider the robots to be points for the sake of simplicity.
We obtain the ATR $\theta_{\phi_{MA}}(\mathbf{x}) = 2.03$ seconds after 43 seconds of computation.
%with the time-traversal constraint as $\dot{\mathbf{h}}(k,i)^{(b)} \geq 0.1$. This increase is due to the consideration of a two-agent specification and the requirement of $9$ B\'ezier segments.
The results are shown in Fig. \ref{fig:ATR} Notice that compared to the illustrative example in Fig. \ref{fig:intro}, the {\em Eventually} specification for agent $1$, $\diamondsuit_{[4,8]}(\mathbf{x}_1 \in B)$, motivates meeting near $A$ in order to maximize the ATR. The formulation will find the optimal meeting place in terms of temporal robustness. 

% Although our method directly allows for efficiently maximizing the ATR, let us consider the constraining of ATR for a fair comparison to the results described in \cite{yu2023efficient}. We use an ATR bound of [0,0], [-1,1], [-2,2], and [-3,3] and report comparative simulation results in Tab. \ref{tab:computational_complexity}. The results of the [-2,2] constraint bound are shown in Fig. \ref{fig:ATR_constrained}

\begin{figure}[t!]
    \centering
    \includegraphics[width=0.5\textwidth]{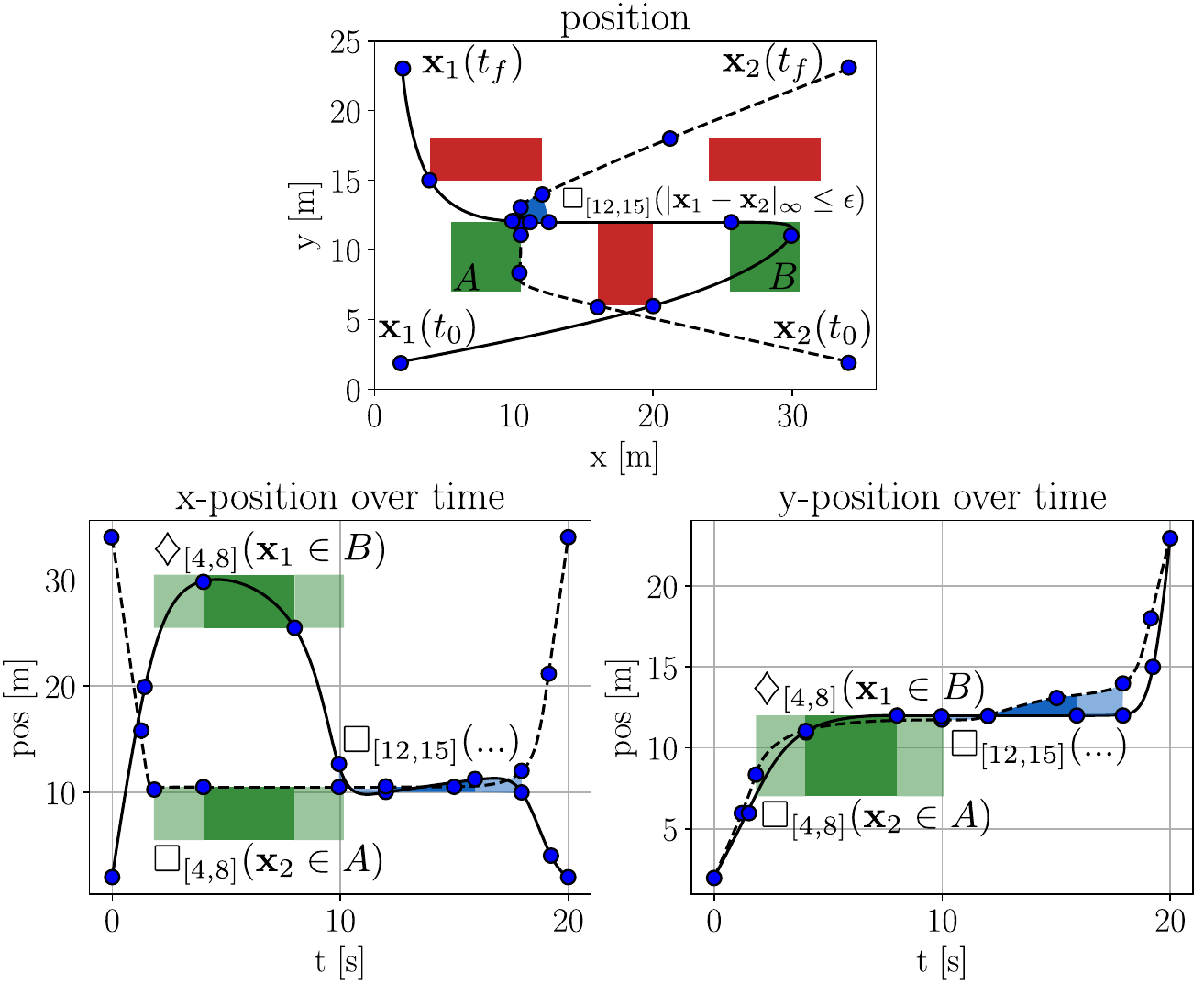}
    \caption{ATR multi-agent example from Eq. \eqref{eq:phi_MA}, Darker shades indicate the predicate (and predicate duration), and lighter shades indicate the ATR of that predicate. Blue indicates $\Box_{[12,15]}(|\mathbf{x}_1 - \mathbf{x}_2|_{\infty} \leq 1)$. The result allows any combinatorial shift of signal $\mathbf{x}_1$ and $\mathbf{x}_2$ up to $|\kappa_i| = \theta_{\phi_{MA}}(\mathbf{x}) = 2.03$.}
    \label{fig:ATR}
    \vspace{-0.5cm}
\end{figure}

%% file: 8_Conclusions.tex
\section{Conclusions}
\label{sec:conclusions}
We have presented a continuous-time motion planner that considers the maximization of the asynchronous temporal robustness of an STL specification. 
Our B\'ezier trajectory parametrization was shown to significantly speed up the robustness maximization of single-agent predicates.
We have shown the soundness of our approach to generating asynchronous temporally robust trajectories and have shown its efficacy in a complex multi-agent scenario.

Future work will involve tightening the under-approximation of the segment predicate ATR while addressing redundancy in the computation. We will also aim to make the MILP encoding more efficient. We further wish to more extensively incorporate real-world constraints and perform real-world experiments on a multi-agent system with an event-triggered replanning strategy.